%% file: neurips_2026_camera_ready.tex
\documentclass{article}

\usepackage[main, final]{neurips_2026}

\usepackage[utf8]{inputenc} 
\usepackage[T1]{fontenc}    
\usepackage{hyperref}       
\usepackage{url}            
\usepackage{booktabs}       
\usepackage{amsfonts}       
\usepackage{nicefrac}       
\usepackage{microtype}      
\usepackage{xcolor}         
\newcommand{\lliu}[1]{{\color{brown}{#1}}}

\usepackage{amsthm}
\theoremstyle{plain}
\newtheorem{theorem}{Theorem}[section]

\newtheorem{lemma}[theorem]{Lemma}

\theoremstyle{definition}
\newtheorem{definition}[theorem]{Definition}

\theoremstyle{remark}

\usepackage{algorithm}
\usepackage{algorithmic}

\usepackage{subcaption}
\usepackage{tabularx}
\usepackage{multirow}
\usepackage{array}
\newcolumntype{Y}{>{\centering\arraybackslash}X}
\newcolumntype{Z}[1]{>{\centering\arraybackslash}p{#1\linewidth}}
\usepackage{wrapfig}

\usepackage{graphicx}
\usepackage{booktabs}
\usepackage{caption}
\usepackage[table]{xcolor}
\usepackage{colortbl}
\definecolor{lightgray}{gray}{0.95}

\usepackage{pifont}
\newcommand{\xmark}{\ding{55}}
\definecolor{checkmark}{HTML}{40826D}
\definecolor{xmark}{HTML}{E62020}

\newcommand{\bccheck}{\large\checkmark}
\newcommand{\bccross}{\large\xmark}

\usepackage{amsmath,bm}
\usepackage{mdframed}
\usepackage{comment}
\usepackage{placeins}

\title{Alleviating Hallucination in Reasoning Tasks with Training-Free Uncertainty-Guided Steering}

\author{%
{Litian Liu$^1$\quad Qiqi Hou$^1$\quad Yubing Jian$^1$\quad Reza Pourreza$^1$}\\ [1pt] \textbf{Mohammad Ghavamzadeh$^1$\quad  Roland Memisevic$^1$\quad Yao Qin$^2$\quad Hong Cai$^1$}\\[5pt] $^1$Qualcomm AI Research\thanks{Qualcomm AI Research is an initiative of Qualcomm Technologies, Inc.} \quad $^2$UC Santa Barbara\\[5pt] {\texttt{litiliu@qti.qualcomm.com}}}

\begin{document}

\maketitle

\begin{abstract}
  Recent work on hallucination detection in large language models has shown that, for a fixed pre-trained model 
  and reasoning task, it is possible to estimate the model’s confidence in the correctness of its outputs. 
  Such uncertainty estimates have primarily been used to improve truthfulness by detecting or filtering confabulations. 
  In this work, we ask whether these signals can instead be used more proactively to directly improve the accuracy 
  of model-generated answers.
  We propose USteer, a simple, training-free \emph{steering} mechanism that adjusts a model’s layer-wise activations during inference 
  using the gradient of a confidence score with respect to the activations. This procedure nudges generation toward outputs 
  with lower uncertainty at inference time, without modifying model parameters or requiring additional 
  supervision. We show that this approach consistently reduces hallucination across a range of reasoning tasks, 
  demonstrating that confidence signals can be leveraged not only for detection, but also for effective 
  inference-time control of model behavior.
\end{abstract}

\vspace{8mm}

\section{Introduction}

Large language models (LLMs) have demonstrated remarkable capabilities in complex reasoning, yet they remain susceptible to generating plausible but incorrect responses, commonly termed hallucinations \citep{ji2023survey,huang2024survey}. 
Consequently, significant research has focused on hallucination detection, using inference-time confidence estimates to signal when a model should abstain from answering \citep{xin2021abstention,abbasi2024conformal,lin2023generating, lin2022towards, manakul2023selfcheckgpt, xiao2021hallucination, kuhn2023semantic, chen2024inside, liu2026out, liu2025enhancing, HouLQAC024, jiang2023calibrating, gao2024spuq}.

While effective for filtering, these post-hoc detection mechanisms are inherently reactive. 
In this work, we propose a shift toward a more proactive paradigm: proactive inference-time steering. 
Rather than merely detecting a hallucination after it has occurred, can we leverage internal uncertainty signals to "steer" the model’s generation process toward more reliable and truthful regions of its representation space?

Recent efforts in steering \citep{li2023inference, zhang2024truthx, zou2023representation, chen2024truth} have attempted this by learning fixed "truthfulness" directions through supervised probes. 
However, these methods rely on static displacement vectors that are applied uniformly across all decoding steps. 
This approach faces a critical bottleneck in the domain of complex reasoning, where the semantic role of intermediate layer representations is highly dynamic and evolves throughout a multi-step chain of thought. Furthermore, these methods require additional training and are sensitive to shifts between the probe's training data and the inference-time distribution.

In this paper, we introduce \textbf{USteer}, a simple, training-free steering framework that dynamically adjusts model activations during the forward pass. 
Our approach is motivated by a key observation: by applying the logit lens to intermediate transformer layers, we find that uncertainty signals predictive of model correctness emerge well before the final layer, particularly the uncertainty signals derived from Neural Collapse Inspired scores (NCI) in \cite{liu2026out}. 
These intermediate layer signals provide real-time feedback on the reliability of the ongoing inference process before the final output is produced.

USteer leverages this discovery by treating the uncertainty signal as a differentiable guidance objective, as shown in Figure~\ref{fig:main}(a). 
At each decoding step, USteer computes the gradient of the confidence score with respect to intermediate layer representations and performs a \emph{single-step} "nudge" toward regions of higher confidence. 
Because this intervention is gradient-based and on-the-fly, it is inherently adaptive to the evolving state of the reasoning chain without requiring any supervised training or parameter updates.

\begin{figure*}[t]
\vspace{2mm}
    \centering
    \begin{subfigure}[t]{.57\textwidth}
        \centering
        \includegraphics[width=\textwidth]{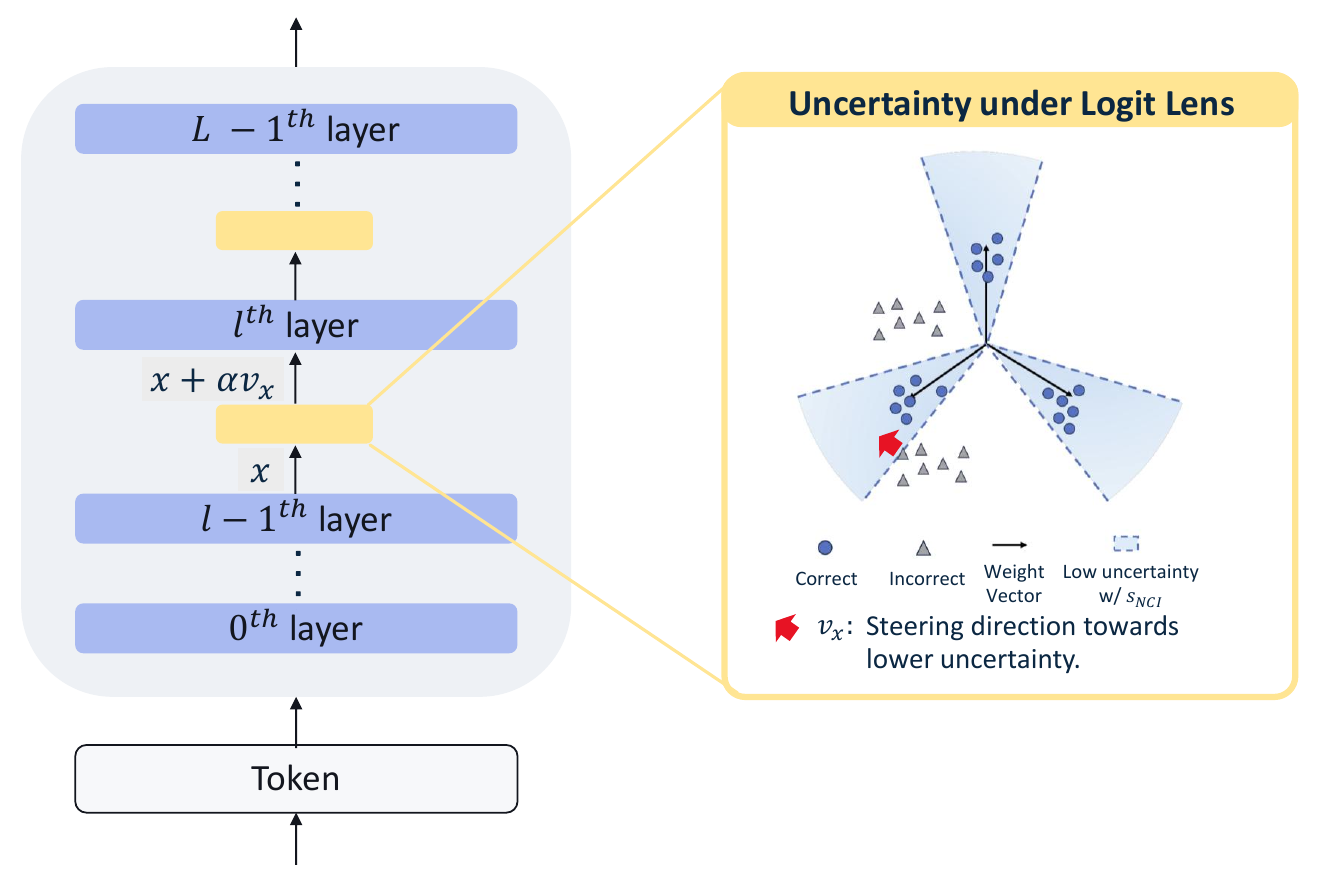}
        \vspace{1mm}
        \caption{Scheme Illustration.}
        \label{fig:sub1}
    \end{subfigure}
    \hfill
    \begin{subfigure}[t]{.42\textwidth}
        \centering
        \includegraphics[width=\textwidth]{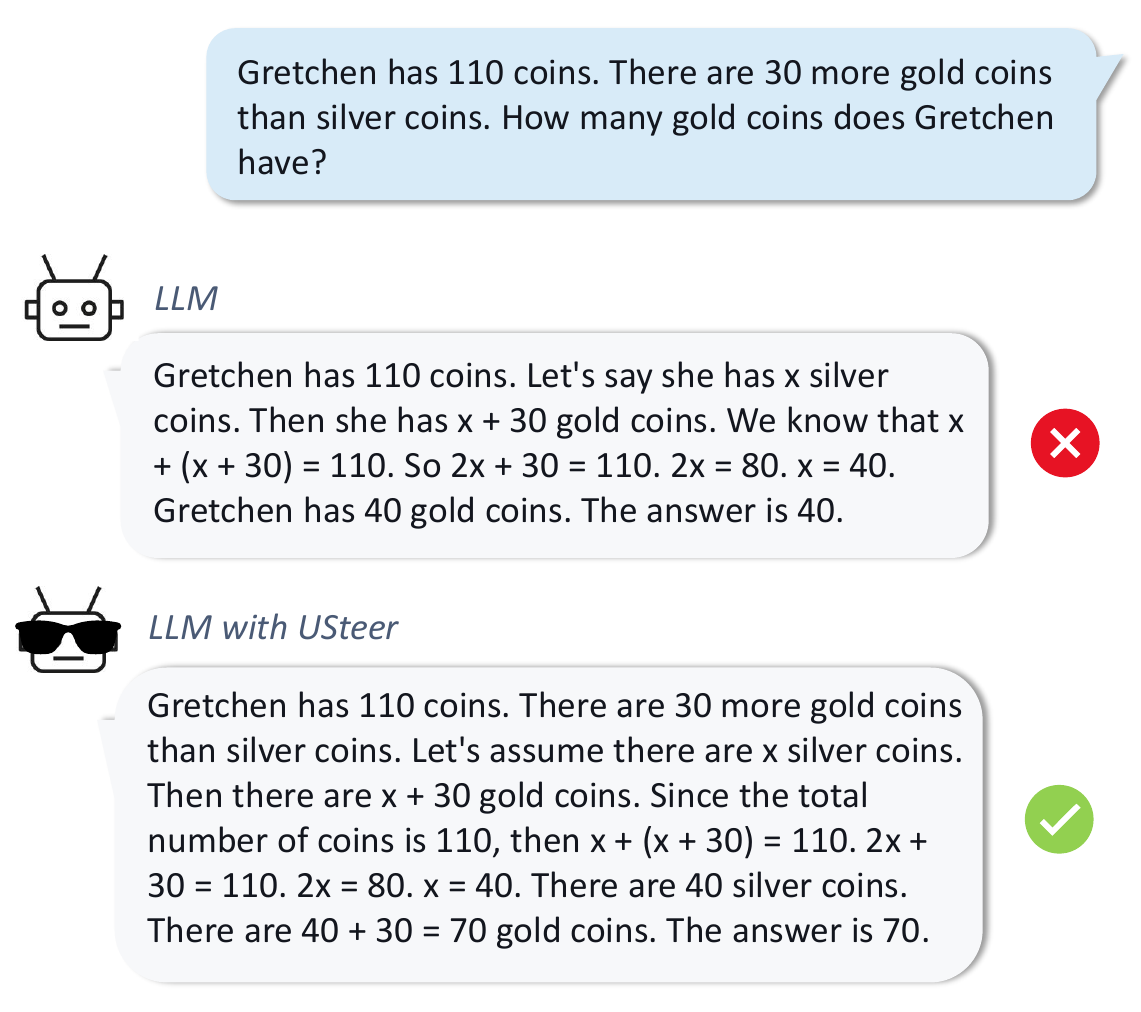}
        \vspace{1mm}
        \caption{Example Outputs.}
        \label{fig:sub2}
    \end{subfigure}
    \vspace{2mm}
    \caption{\textbf{Overview of USteer, a training-free uncertainty-guided steering method.} 
    (a) \textit{Scheme Illustration.}
    We intervene on output of intermediate transformer layers, with the steering mechanism highlighted in yellow.
    USteer is training-free, where steering directions are derived from uncertainty signals defined by the language modeling head.
    Under the logit lens, we evaluate uncertainty scores based on NCI~\citep{liu2026out}) and steer representations toward regions associated with lower uncertainty.
    (b) \textit{Example Outputs.}
    USteer reduces hallucination, as illustrated on an example from the GSM8K dataset using \texttt{Llama-3.2-1B-Instruct}. 
    }
    \label{fig:main}
\end{figure*}

We evaluate USteer across diverse reasoning benchmarks, including CommonsenseQA\citep{talmor-etal-2019-commonsenseqa}, StrategyQA\citep{geva2021did}, GSM8K\citep{cobbe2021training}, and AQuA\citep{ling2017program}, using multiple model families, including Llama-3.2, Qwen-2.5, and Qwen-3. 
Our results show that USteer consistently improves accuracy while reducing hallucinations in reasoning tasks across datasets and models. 
Moreover, USteer incurs negligible inference-time latency overhead.
Overall, these results demonstrate that USteer provides a practical framework for uncertainty-guided inference-time control in large language models.

\section{Related Work}

Among inference-time approaches for modifying large language model behavior, our work belongs to the class of steering methods, also known as activation editing. 
Prior work has explored activation steering across a variety of applications. 
For example, \citet{subramani2022extracting} and \citet{turner2023steering} demonstrate that steering can induce sentiment transfer in language models, noting that these effects are primarily mediated by higher layers. Moreover, \citet{mahmoud2025improving} show that steering can enhance cross-lingual understanding by shifting representations of non-English tokens toward their English counterparts. Recent work has also introduced steering to mitigate safety risks. For instance, \citet{wang2025steering} employ steering to counter jailbreaking and toxic generation by directing representations away from adversarial feature spaces.

Closest to our work are methods leveraging steering for hallucination mitigation in knowledge-intensive tasks~\citep{li2023inference, zhang2024truthx, zou2023representation, chen2024truth}. 
These approaches typically train supervised probes to extract a static "truthfulness" direction, applying a fixed displacement vector across all decoding steps. 
However, such static interventions face challenges in complex multi-step reasoning tasks, where intermediate representations evolve dynamically. 
To bridge this gap, USteer introduces a training-free, dynamic steering paradigm, leveraging model-internal uncertainty signals from hallucination detection literature~\citep{liu2026out} on-the-fly. 
By dynamically nudging representations toward more confident regions at each step, USteer uses confidence to control LLM behavior~\citep{li2026confidence} and improve reasoning accuracy.

Another line of work improves LLM truthfulness through contrastive decoding, which contrasts the predictions of an expert model with those of an amateur model to suppress linguistic biases and amplify factual signals~\citep{li2023contrastive, kai2024sh2, zhang2025alleviating}. 
DoLa~\citep{chuang2023dola} applies this idea within a single model, contrasting the full model's output with predictions from earlier-layer exits to emphasize knowledge that emerges in deeper layers. 
Similarly, our work views the layers of a single LLM as an ensemble. However, rather than contrasting logits across layers for token selection, we use uncertainty signals that emerge in early layers to dynamically steer the model's hidden representations.

\section{Methodology}\label{sec:method}

In this section, we introduce USteer, a training-free steering method that guides intermediate layer representations toward regions of higher confidence to mitigate hallucinations in reasoning tasks.
We first introduce the uncertainty metric, then provide empirical evidence that uncertainty signals emerge in intermediate layers, and finally describe the steering algorithm.

\subsection{Setup}

Consider a Large Language Model (LLM) $f$ with model dimension $d_{\text{model}}$ and vocabulary $\mathcal{V}$. 
At each decoding step, given the input sequence $\bm{x}$ and previously generated tokens, the model produces hidden representation $\bm{z} \in \mathbb{R}^{d_{\text{model}}}$ at the penultimate layer. 
This representation is mapped by the language head $f_{\text{head}}$ to produce logits over $\mathcal{V}$.
The language head functions as a linear classifier where the probability of the next token is determined by the proximity of $\bm{z}$ to a set of learned weight vectors $\{\bm{w}_v\}_{v \in \mathcal{V}}$, corresponding to the unembedding weights of the linear output layer.
In state-of-the-art models, the language head is typically zero-biased. 
Under greedy decoding, the predicted token $\hat{c}$ is identified as:
\vspace{3mm}
\begin{equation}
    \hat{c} = \arg\max_{v \in \mathcal{V}} (\bm{w}_v^\top \bm{z})
\end{equation}

To quantify the uncertainty associated with this prediction, we adapt the Neural Collapse Inspired (NCI) score, which was first introduced for out-of-distribution detection in classification tasks by \cite{liu2025detecting} and subsequently adapted for post-hoc hallucination detection in LLM reasoning tasks by \cite{liu2026out}. 
The NCI score has been shown to be highly effective for detecting hallucinations in reasoning tasks while remaining computationally inexpensive, making it particularly suitable for inference-time steering for hallucination mitigation.
Specifically, NCI measures the geometric alignment between a hidden state and the weight vector of the chosen token, as illustrated in Figure~\ref{fig:main}. 
To facilitate the gradient-based steering introduced later in this paper, we utilize a squared, simplified version of the score:
\begin{definition}[Confidence Guidance Score]\label{def:nci}  
\textit{Adapted from \cite{liu2026out}.} 
Given a hidden representation $\bm{z}$ and the weight vector $\bm{w}_{\hat{c}}$ corresponding to the most-likely token $\hat{c}$, the confidence guidance score $g(\bm{z})$ is defined as the squared cosine similarity:
\vspace{2mm}
\begin{equation}\label{eq:generalScore}
g(\bm{z}) = \frac{(\bm{w}_{\hat{c}}^\top \bm{z})^2}{|\bm{z}|_2^2},
\end{equation}
where $\hat{c} = \arg\max_{v \in \mathcal{V}} (\bm{w}_v^\top\bm{z})$ is the most likely token given the hidden representation $\bm{z}$. 
\end{definition} 
\vspace{1mm}
Geometrically, a higher $g(\bm{z})$ indicates that the representation is well-aligned with the weight vector of the most-likely token, signifying high confidence. 
Conversely, a lower score suggests the representation resides in a region of lower confidence, which serves as a precursor to hallucination.

In the remainder of this section, we develop our steering algorithm using NCI as the confidence guidance score. 
In Section~\ref{sec:ablation_metric}, we further evaluate alternative uncertainty scores, demonstrating the generalizability of proposed framework.

\subsection{Observation: Uncertainty Signal Emerges at Intermediate Layers.}

While uncertainty is conventionally measured at the final layer, we hypothesize that intermediate layer representations also contain uncertainty signals which are predictive of answer correctness. 
This hypothesis is rooted in the architecture of modern LLMs: due to the prevalence of residual connections, the model can be viewed as a “weak-to-strong” ensemble, in which predictions obtained from intermediate layer representation serve as proxies for the model’s final generation. 
Uncertainty signals derived there may therefore provide early indications of the reliability of the final prediction.

To test this hypothesis, we employ the logit lens, i.e., we project the intermediate layer representations, given by the activations $\bm{x}^l$, into the vocabulary space using the pretrained final language head. 
In particular, since state-of-the-art models typically apply a normalization step, such as RMSNorm, immediately before the language head to ensure numerical stability and proper feature scaling, we mirror this processing to ensure that the intermediate layer representations are compatible with the language head’s weights.

Specifically, for the activation $\bm{x}^l$ at layer $l$, we compute the normalized representation $\bm{z}^l$ as follows:

\begin{equation}\label{eq:scores}
\bm{z}^{l} = \bm{D} \frac{\bm{x}^{l}}{\mathrm{RMS}(\bm{x}^{l})}
\end{equation}

where $\bm{D}$ is the diagonal scaling matrix (the learned gain parameters) from the model's final normalization layer and RMS represents the root mean square across the hidden dimension $d_{model}$. 

\begin{equation}
\mathrm{RMS}(\bm{x}) = \sqrt{\frac{1}{d_{model}} \sum_{i=1}^{d} x_{i}^2 + \epsilon}.
\end{equation}

Using this normalized representation, we identify the most likely token at layer $l$ under the logit lens:
\vspace{2mm}
\begin{equation}
    \hat{c}_l = \arg\max_{v \in \mathcal{V}} (\bm{w}_v^\top \bm{z}^l)
\end{equation}
Following Definition~\ref{def:nci}, we compute the uncertainty guidance score $g(\bm{z}^l)$ at each intermediate layer using the corresponding layer-specific prediction $\hat{c}_l$.

\begin{figure*}[t]
\begin{center}
\vspace{2mm}
\includegraphics[width=\textwidth]{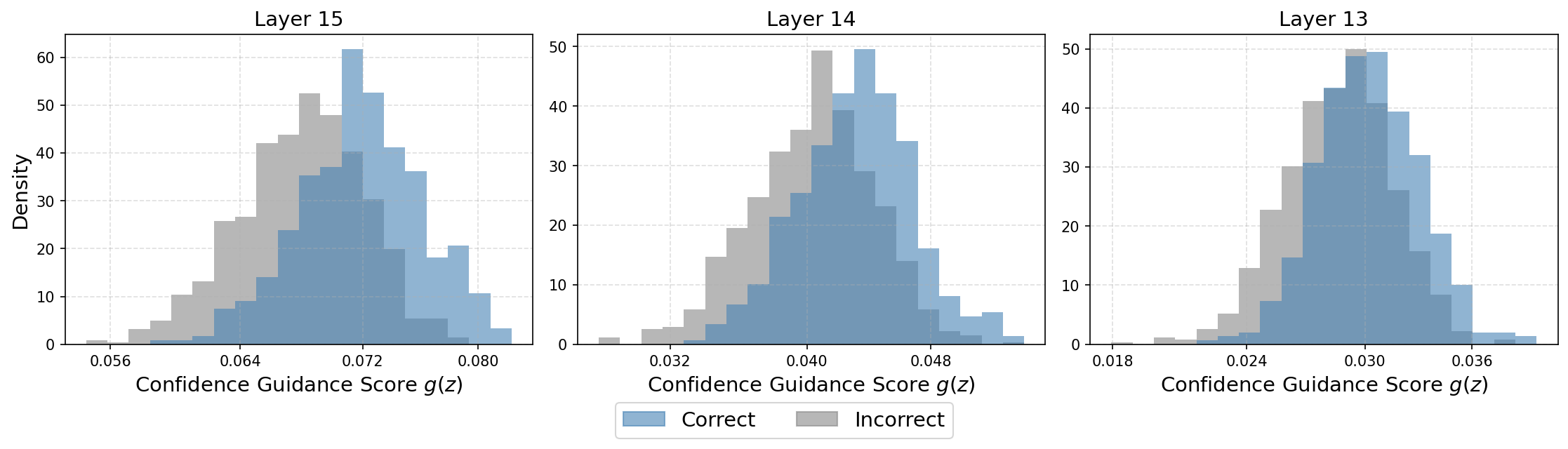}
\vspace{1mm}
\caption{
\textbf{Uncertainty Signals Emerge at Intermediate Layers of LLM.}
For each layer, we compute the score from Equation~\ref{eq:generalScore} across steps under the logit lens and compare its distribution for non-hallucinated (correct) and hallucinated (incorrect) examples.
Correct intermediate embeddings tends to exhibit higher confidence guidance score.
While the distinction between correct and incorrect is strongest at the final layer (Layer 15), intermediate layers (Layer 14 and Layer 13) already show such trend.
This demonstrates that the language head provides meaningful uncertainty signals before the final layer, motivating steering at intermediate layer representations.
Evaluation is conducted on GSM8K using \texttt{Llama-3.2-1B-Instruct}.
}\label{fig:observation}
\end{center}
\end{figure*}

\vspace{2mm}

\textbf{Empirical Evidence:} In Figure~\ref{fig:observation}, 
we analyze the distributions of $g(\bm{z}^l)$ computed from intermediate layer representations. Although the distinction between hallucinated and non-hallucinated predictions is most pronounced at the final layer, a similar trend emerges several layers earlier. Specifically, representations associated with non-hallucinated generations tend to exhibit greater alignment with their corresponding weight vectors than those associated with hallucinated generations, indicating higher confidence.  In Appendix~\ref{app:auroc}, we further demonstrate quantitatively that confidence scores from intermediate layers are predictive of answer correctness, as measured by the area under the receiver operating characteristic curve (AUROC). These early uncertainty signals provide guidance for steering the model toward more confident representations during inference, as described in the following section.


\subsection{\emph{USteer:} Uncertainty-guided Steering at Intermediate Layers.}

The goal of USteer is to dynamically intervene on the intermediate layer representation $\bm{x}^l$ such that its normalized counterpart $\bm{z}^l$ is nudged toward a region of higher confidence. For maximum efficiency, we achieve this through a single-step gradient ascent update during the forward pass.

To find the optimal steering direction, we first calculate the gradient of the confidence guidance score in the normalized space where the language head resides.

\vspace{3mm}
\begin{lemma}[Gradient of the Confidence Guidance Score]
\label{lemma:gradient}
Let $g(\bm{z}) = \frac{(\bm{w}_{\hat{c}}^\top \bm{z})^2}{\|\bm{z}\|_2^2}$ be the confidence guidance score where $\bm{w}_{\hat{c}}$ is the target weight vector. 
The gradient with respect to the normalized representation $\bm{z}$ is given by:
\vspace{2mm}
\begin{equation}\label{eq:gradient_z}
\nabla_{\bm{z}} g(\bm{z}) = \frac{2(\bm{w}_{\hat{c}}^\top \bm{z})}{\|\bm{z}\|_2^2} \left( \bm{w}_{\hat{c}} - \frac{\bm{w}_{\hat{c}}^\top \bm{z}}{\|\bm{z}\|_2^2} \bm{z} \right)
\end{equation}
\end{lemma}
\vspace{2mm}

\begin{proof}
    See Appendix~\ref{app:proof}. 
\end{proof}

While the confidence guidance score is derived from the normalized state $\bm{z}^l$, the intervention must occur in the pre-normalization space $\bm{x}^l$ to remain consistent with the model’s computational graph. We propagate the gradient back through the RMSNorm layer using the chain rule:
\begin{equation}\label{eq:gradient_x}
\vspace{1mm}
\begin{aligned}
    \nabla_{\bm{x}^{l}} g 
    &= \left( \frac{\partial \bm{z}^{l}}{\partial \bm{x}^{l}} \right)^\top \nabla_{\bm{z}^{l}} g \\
    &= \frac{1}{\|\bm{x}^{l}\|_2}
    \left( I - \frac{\bm{x}^{l} (\bm{x}^{l})^\top}{\|\bm{x}^{l}\|_2^2} \right)
    \bm{D} \, \nabla_{\bm{z}^{l}} g,
\end{aligned}
\end{equation}

The representation is then updated along the normalized gradient direction:
\begin{equation}
\vspace{1mm}
\bm{x}^l \leftarrow \bm{x}^l + \alpha \frac{\nabla_{\bm{x}^l} g}{\|\nabla_{\bm{x}^l} g\|_2}\end{equation}where $\alpha$ is a hyperparameter controlling the steering strength. 
The full procedure is summarized in Algorithm \ref{alg:USteer}.

\vspace{2mm}
\begin{algorithm}
\caption{\emph{USteer:} Training-Free Uncertainty-Guided Steering}\label{alg:USteer}
\begin{algorithmic}[1]
\STATE \textbf{input:} input prompt $\bm{p}$, steering strength $\alpha$,  steering layers $\mathcal{L}$. 
\STATE \textbf{output:} Model response $\bm{y} = [y_0, y_1, \dots, y_T]$. 
    \FOR{each decoding step $t$}
        \FOR{each layer $l$}
                \STATE Compute $\bm{x}^l$ using the potentially steered prior layer representations.
                \IF{\( l \in \mathcal{L} \)}
                    \STATE Steer layer output: $\bm{x}^l \leftarrow  \bm{x}^l + \alpha \frac{\nabla_{\bm{x}^l} g}{\|\nabla_{\bm{x}^l} g\|_2}$. 
                \ENDIF
        \ENDFOR
        \STATE Decode token $y_t$ from the final hidden representation under the chosen decoding scheme. 
    \ENDFOR
\RETURN Model response $\bm{y} = [y_0, y_1, \dots, y_T]$. 
\end{algorithmic}
\end{algorithm}

\section{Experiments}

In this section, we comprehensively evaluate USteer across diverse reasoning datasets and model architectures. 
Our empirical results demonstrate that USteer consistently mitigates hallucinations and improves reasoning accuracy, with the largest gains observed in low-confidence regions. 
We further conduct extensive ablation studies to examine the effects of key hyperparameters, including the selection of steering layers and the steering strength. 
Finally, we evaluate the robustness of USteer under stochastic decoding and demonstrate its compatibility with alternative confidence guidance scores.



\subsection{Main Results}\label{sec:main_results}

In Table~\ref{tab:main_results}, we show compare our uncertainty-based steering, 
USteer, against baselines. 
Our method consistently mitigates hallucination across datasets and models, improving reasoning accuracy with negligible latency overhead.

\paragraph{Datasets}
We consider both commonsense and mathematical reasoning tasks. 
For commonsense reasoning, we evaluate on CommonsenseQA (CSQA) \citep{talmor-etal-2019-commonsenseqa}, which assesses commonsense world knowledge in a multiple-choice format, and StrategyQA \citep{geva2021did}, which requires multi-hop reasoning for binary (yes/no) questions. 
For mathematical reasoning, we evaluate on GSM8K \citep{cobbe2021training}, which requires free-form numerical answers, and AQuA \citep{ling2017program}, which uses a multiple-choice format.
These datasets span diverse answer formats, including multiple-choice, binary, and free-form numerical outputs (see Appendix~\ref{app:implementation} for examples). 
We evaluate on the CSQA validation set (1{,}221 questions), the StrategyQA dev set (229 questions), the GSM8K test set (1{,}319 questions), and the AQuA validation split (254 questions).
We report 95\% bootstrap confidence intervals for the main results in Appendix~\ref{app:greedy_ci} to account for uncertainty due to finite sample sizes.

\paragraph{Models}
We evaluate on \texttt{Llama-3.2-1B-Instruct} \citep{grattafiori2024llama} and
\texttt{Qwen-2.5-3B-Instruct} \citep{qwen2025qwen25technicalreport} to assess
generality across model families.
We further demonstrate scalability to larger models by evaluating
\texttt{Qwen-2.5-14B-Instruct} and \texttt{Qwen-3-32B} \citep{yang2025qwen3technicalreport} in Appendix~\ref{app:large_model}.

\paragraph{Baselines}
Alongside standard decoding, we compare USteer against three competitive inference-time baselines, namely DoLa \citep{chuang2023dola}, ITI \citep{li2023inference}, and TruthX \citep{zhang2024truthx}. 
DoLa represents the contrastive decoding paradigm, contrasting logits from the final layer against early-exit results to alleviate hallucination during decoding. 
In contrast, ITI and TruthX are training-based steering algorithms that rely on supervised probes to find a static direction, applying the exact same displacement vector at every decoding step. 
Mechanically, ITI applies these edits exclusively to attention heads, whereas TruthX argues that attention-level editing is insufficient and extends the intervention to all internal hidden representations.

\paragraph{Hyperparameter Selection}
We select the same steering layers across all datasets for each model architecture. Because uncertainty guidance tends to be more pronounced in the upper layers, we restrict the candidate layers to the final $d$ intermediate layers of the network. 
Specifically, for \texttt{Llama-3.2-1B-Instruct}, we apply steering to the last 3 intermediate layers, while for \texttt{Qwen-2.5-3B-Instruct}, we target the last 8 intermediate layers. 

The steering strength $\alpha$ is optimized per dataset using a validation set, sweeping values from $\{0.1, 0.2, \dots, 0.9\}$. For \texttt{Llama-3.2-1B-Instruct}, the optimal strength $\alpha$ is set to $0.2$ for CSQA, $0.1$ for StrategyQA, and $0.3$ for GSM8K, while a smaller grid search yields $\alpha = 0.03$ for the more sensitive AQuA benchmark. For \texttt{Qwen-2.5-3B-Instruct}, $\alpha$ is selected as $0.9, 0.3, 0.5,$ and $0.3$ for CSQA, StrategyQA, GSM8K, and AQuA, respectively.

\paragraph{Overall Performance} 

In Table~\ref{tab:main_results}, we evaluate our steering method, USteer, alongside competitive inference-time baselines under greedy decoding ($T = 0$). 
The empirical results reveal that the training-based steering methods, ITI and TruthX, face challenges in mitigating hallucination, sometimes degrading reasoning accuracy compared to the base model. 
Such performance aligns with our intuition that steering along a static, pre-computed direction is ill-suited to accommodate the highly dynamic and evolving nature of representation spaces during a multi-step reasoning chain. 
Conversely, USteer consistently improves reasoning accuracy across all evaluated datasets and model architectures, demonstrating that leveraging dynamic, on-the-fly uncertainty signals provides a more robust and adaptive guidance mechanism for complex logical tasks.

In addition to task accuracy, we evaluate the computational efficiency of each method by comparing their inference latencies. 
Latency is reported in milliseconds per token (where lower is better) and measured end-to-end on the StrategyQA dataset using a single NVIDIA A100-80GB GPU for each model\footnote{Reported latencies may vary due to resource contention on shared compute nodes.}. 
Because USteer restricts its interventions to a select subset of upper intermediate layers and introduces minimal arithmetic computation per layer, it adds negligible overhead to the forward pass. 
Consequently, our method closely maintains the latency profile of standard decoding.

\vspace{5mm}
\begin{table*}[t]
\caption{
\textbf{USteer mitigates hallucination in reasoning tasks across datasets and models with negligible latency overhead.}
Performance is reported in accuracy and latency
(ms/token; lower is better).
The best performance is shown in \textbf{bold}, and the second-best performance is \underline{underlined}.
}
\vspace{2mm}
\centering
\begin{tabularx}{0.99\linewidth}{
>{\centering\arraybackslash}X
Z{0.14}
Z{0.1}
Z{0.1}
Z{0.1}
Z{0.1}
Z{0.1}
}
\toprule
Methods
& Training-free
& Latency $\downarrow$
& CSQA
& StrategyQA
& GSM8K
& AQuA \\
\midrule

\multicolumn{7}{>{\centering\arraybackslash}p{\linewidth}}{
\textit{Model: Llama-3.2-1B-Instruct}
} \\[4pt]

Standard
& -
& 12.5
& 50.94
& 58.52
& 34.19
& \underline{37.80} \\

DoLa
& \bccheck
& 17.2
& 51.19
& 60.70
& \textbf{36.39}
& 31.50 \\

ITI
& \bccross
& 22.1
& \textbf{54.46}
& 59.39
& 35.56
& \underline{37.80} \\

TruthX
& \bccross
& 21.9
& 53.07
& \underline{61.57}
& 34.87
& 36.22 \\

\rowcolor{lightgray}
USteer (Ours)
& \bccheck
& 12.8
& \underline{54.05}
& \textbf{63.32}
& \underline{36.24}
& \textbf{39.37} \\

\midrule

\multicolumn{7}{>{\centering\arraybackslash}p{\linewidth}}{
\textit{Model: Qwen-2.5-3B-Instruct}
} \\[4pt]

Standard
& -
& 21.5
& 77.72
& \underline{68.56}
& 79.91
& 54.72 \\

DoLa
& \bccheck
& 25.8
& 78.54
& 66.81
& 78.85
& 54.72 \\

ITI
& \bccross
& 47.7
& 77.72
& 68.12
& \underline{80.89}
& \textbf{57.09} \\

TruthX
& \bccross
& 50.2
& \underline{78.62}
& 67.25
& 80.06
& \underline{55.12} \\

\rowcolor{lightgray}
USteer (Ours)
& \bccheck
& 22.8
& \textbf{79.20}
& \textbf{70.74}
& \textbf{80.97}
& \underline{55.12} \\

\bottomrule
\end{tabularx}
\label{tab:main_results}
\vspace{2mm}
\end{table*}

\subsection{Performance across Confidence Regions}

By explicitly guiding generation toward more confident regions, USteer is particularly effective at correcting low-confidence errors. 
Meanwhile, it can be less effective for confidently incorrect predictions, as steering toward more confident representations may reinforce an existing error. 

We quantitatively examine this behavior on \texttt{Llama-3.2-1B-Instruct} in Table~\ref{tab:confidence_regions}. 
Specifically, we group examples into percentiles based on their per-token softmax confidence within each dataset and then pool the corresponding confidence regions across the four datasets. 
USteer achieves larger net accuracy gains over greedy decoding than existing methods, with its advantage becoming more pronounced toward the lowest-confidence tail.

\begin{table}[t] \centering \caption{ \textbf{USteer provides larger accuracy gains in low-confidence regions.} Confidence regions are pooled across CSQA, StrategyQA, GSM8K, and AQuA. Net accuracy gains are reported in percentage points over standard greedy decoding on \texttt{Llama-3.2-1B-Instruct}. } \label{tab:confidence_regions}
\vspace{2mm}
\renewcommand{\arraystretch}{1.2} 
\setlength{\tabcolsep}{10pt}
\begin{tabular}{lcccc} \toprule Confidence Region & USteer & DoLa & ITI & TruthX \\ \midrule All examples & +2.6 & +0.7 & +2.1 & +1.3 \\ Bottom 25\% & +6.7 & +3.8 & +5.1 & +3.3 \\ Bottom 5\% & +14.4 & +8.5 & +9.8 & +9.8 \\ \bottomrule \end{tabular} \end{table}

\subsection{Effect of Steering Layers}

\begin{figure*}[t]
    \centering
    \begin{minipage}[t]{.49\textwidth}
        \centering
        \includegraphics[width=\textwidth]{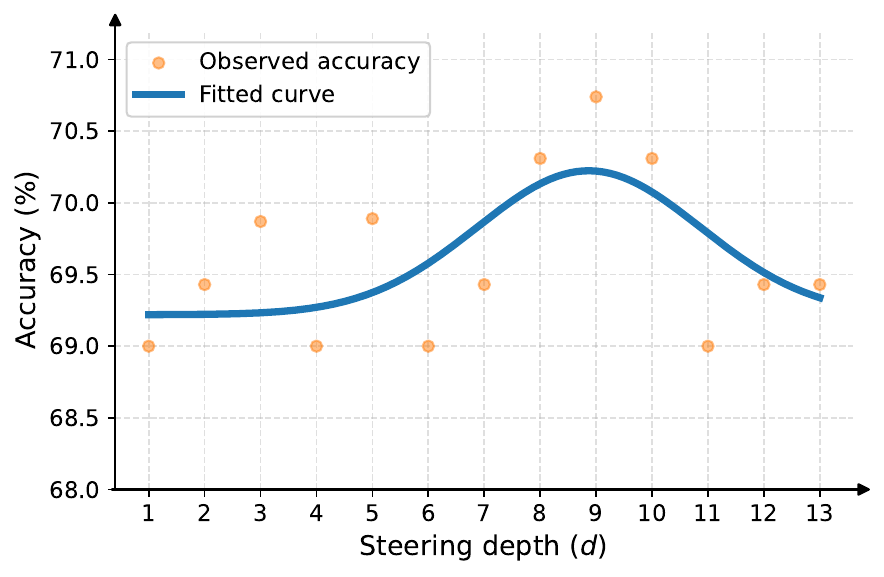}
        \caption{\textbf{Trade-off in Selecting Steering Layers}, reflecting competing effects of steering strength and guidance reliability.
        The x-axis presents steering depth {$d$}.  
        The y-axis presents task accuracy. 
        Experiments on StrategyQA using \texttt{Qwen-2.5-3B-Instruct}.}
        \label{fig:layer-tradeoff}
    \end{minipage}
    \hfill
    \begin{minipage}[t]{.49\textwidth}
        \centering
        \includegraphics[width=\textwidth]{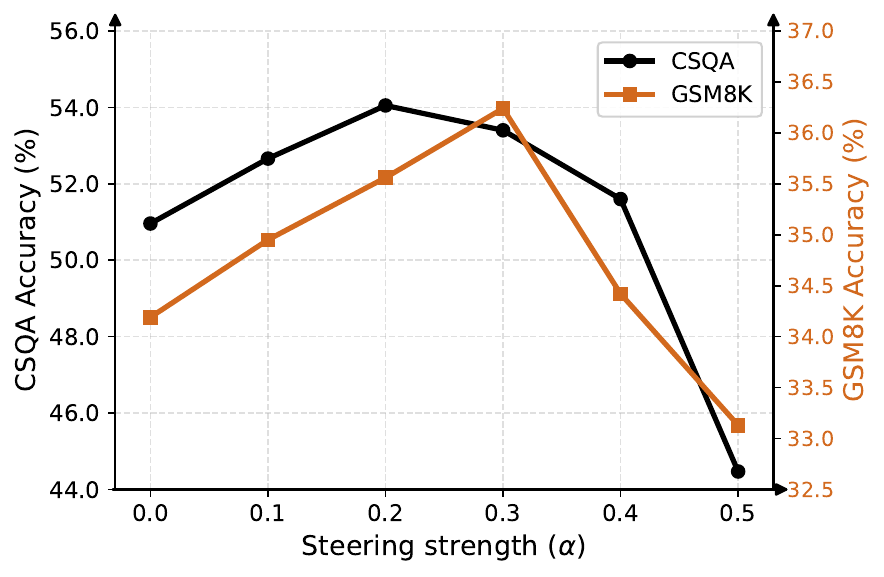}
        \caption{\textbf{USteer achieves robust performance across varying steering strengths}.
        The x-axis presents steering strength $\alpha$.  
        The y-axis present task accuracy for CSQA and GSM8K.  
        Experiments using \texttt{Llama-3.2-1B-Instruct}. }
        \label{fig:effect_strength}
    \end{minipage}
\end{figure*}

To study the effect of layer selection, we conduct experiments on StrategyQA using \texttt{Qwen-2.5-3B-Instruct}, which consists of 36 transformer layers. 
We vary the layers at which steering is applied and report task accuracy.

\vspace{2mm}
\begin{table}[t]
\centering
\caption{\textbf{USteer consistently mitigates hallucination across different selections of steering layers.}
Experiments on StrategyQA using \texttt{Qwen-2.5-3B-Instruct}. 
Accuracy reported. 
USteer improves performance when applied at lower, middle, and upper layers.
}
\vspace{3mm}
\label{tab:layer_overall_effects}
\begin{tabular}{lcccc}
\toprule
 & Standard & Lower Layers & Middle Layers & Upper Layers \\
\midrule
Accuracy (\%) & 68.56 & 69.00 & 69.43 & 69.43 \\
\bottomrule
\end{tabular}
\end{table}
\vspace{4mm}

\paragraph{Steering remains effective across different layer choices.}
We first study the overall effect of layer choice by partitioning the model into lower layers (0--11), middle layers (11--22), and upper layers (23--34). 
For each setting, we report the best accuracy over steering magnitudes $\alpha \in \{0.1, 0.2, \dots, 0.9\}$. 
As shown in Table~\ref{tab:layer_overall_effects}, applying steering within any of these layer regions consistently improves accuracy over standard decoding, demonstrating the overall robustness of uncertainty-guided steering across the transformer stack.

Despite this broad effectiveness, the magnitude of improvement varies systematically across depth. 
In particular, steering at upper layers yields larger gains than steering at lower layers. 
This trend aligns with our earlier analysis in Section~\ref{sec:method}, where uncertainty signals extracted through the logit lens exhibit stronger distinctions between correct and hallucinated generations in higher layers. 
This empirical finding justifies our heuristic of restricting the candidate steering layers to the final $d$ intermediate layers of the network.

\paragraph{Layer selection reflects a trade-off.}
In Figure~\ref{fig:layer-tradeoff}, we study the effect of steering depth {$d$} by sweeping $d$ and reporting the best accuracy over steering magnitudes $\alpha \in \{0.1, 0.2, \dots, 0.9\}$. 
We observe a trade-off across layers. 
On one hand, steering deeper layers (larger $d$) exerts a stronger and more direct influence on the final model output. 
On the other hand, uncertainty estimates obtained through the logit lens become less reliable at lower layers, where representations are less aligned with the final prediction space. 
These competing effects lead to an approximately bell-shaped performance trend as a function of steering depth. 
Nevertheless, all tested steering depths consistently outperform standard decoding (68.56 accuracy).
Finally, the optimal layer selection appears to be model-specific. 
Although a single steering depth generalizes reasonably well across datasets for a fixed model in Section~\ref{sec:main_results}, the best-performing depth varies across model families, likely reflecting differences in representational geometry and layer-wise feature evolution.

\subsection{Effect of Steering Strength}

To study the effect of steering strength $\alpha$, we conduct experiments on CSQA and GSM8K using \texttt{Llama-3.2-1B-Instruct}. 
Experimental settings, including steering layer selection, follow Section~\ref{sec:main_results}. 
In Figure~\ref{fig:effect_strength}, we plot task accuracy on both reasoning tasks as a function of $\alpha \in \{0, 0.1, 0.2, 0.3, 0.4, 0.5\}$.

We observe a trade-off in steering strength. 
When $\alpha$ is too small, the intervention has only a limited effect on the hidden representations, resulting in relatively modest performance gains over standard decoding. 
As $\alpha$ increases, steering becomes more effective and accuracy improves accordingly. 
However, excessively large steering magnitudes (e.g., $\alpha = 0.5$) begin to degrade performance, likely because overly aggressive interventions distort the underlying representation geometry and disrupt the model’s reasoning process.
Importantly, performance varies smoothly as a function of $\alpha$ for both tasks. 
A broad range of steering strengths consistently outperforms standard decoding ($\alpha = 0$), suggesting that USteer is reasonably robust to hyperparameter selection. 


\subsection{Effect of Stochastic Decoding}\label{sec:stochastic}

\begin{table}[t]
\centering
\caption{\textbf{USteer mitigates hallucination under stochastic decoding}, consistently yielding gains across a range of temperatures with statistical significance.
Results are reported on CSQA using \texttt{Llama-3.2-1B-Instruct}, with accuracy shown as the mean and 95\% confidence intervals.
}
\vspace{2mm}
\label{tab:stochastic_result}
\begin{tabular}{lcccc}
\toprule
 & temp = 0.2 & temp = 0.5 & temp = 0.8 & temp = 1.0 \\
\midrule
Standard       & $50.38 \pm 0.46$ & $50.22 \pm 0.29$ & $47.83 \pm 0.32$ & $45.44 \pm 0.38$  \\
USteer (ours)            & $53.55\pm 0.38$ & $52.22 \pm 0.42$ & $50.37 \pm 0.19$ & $48.78 \pm 0.32$ \\
\bottomrule
\end{tabular}
\end{table}
\vspace{3mm}

So far, we have evaluated USteer under greedy decoding (temperature $= 0$). 
We now examine whether the proposed steering mechanism remains effective under stochastic decoding, where randomness introduced during sampling can substantially alter the generation trajectory.
In Table~\ref{tab:stochastic_result}, we evaluate USteer on CommonsenseQA (CSQA) using \texttt{Llama-3.2-1B-Instruct} across a range of sampling temperatures $temp \in \{0.2, 0.5, 0.8, 1.0\}$. 
The remainder of the experimental setup follows Section~\ref{sec:main_results}. 
For each temperature, we perform five independent runs with different random seeds and report the mean accuracy together with 95\% confidence intervals.

As shown in the table, USteer consistently improves over standard stochastic decoding across all temperature settings. 
Although the absolute task accuracy varies with temperature, the performance gains introduced by steering remain stable and consistent. 
Moreover, the corresponding 95\% confidence intervals are non-overlapping in most settings, indicating that the improvements are statistically significant. 
These results suggest that uncertainty-guided steering is compatible with stochastic decoding and can provide reliable benefits subject to sampling noise.
Additional analyses of generation behavior under stochastic decoding are provided in Appendix~\ref{app:generation_behavior}.

\subsection{Effect of Alternative Confidence Guidance Scores} \label{sec:ablation_metric} 

So far, we have used NCI variant (Definition~\ref{def:nci}) as the confidence guidance score for USteer given its computational efficiency and effectiveness in hallucination detection in reasoning tasks. 
We now examine the general applicability of USteer by considering two additional scores: the standard softmax confidence and fDBD~\cite{liu2026out}.
Similar to NCI, fDBD was originally introduced for out-of-distribution detection in classification tasks~\citep{liu2024fast} and later adapted for hallucination detection in reasoning tasks~\cite{liu2026out}. 
Following Definition~\ref{def:nci}, we consider a squared, simplified version of fDBD score with $k = 1000$, as detailed in Appendix~\ref{app:fdbd_score}. 

Following the setup in Section~\ref{sec:main_results}, we USteer with alternative scores on \texttt{Qwen-2.5-3B-Instruct}.
For softmax confidence, we use $\alpha=1.1$ for CSQA and $\alpha=0.1$ for StrategyQA, GSM8K, and AQuA.
For fDBD, we use $\alpha=0.5$ for CSQA, StrategyQA, and GSM8K, and $\alpha=0.08$ for AQuA. 
As shown in Table~\ref{tab:alternative_metric}, both alternative confidence guidance scores improve accuracy over standard inference across all datasets, demonstrating the applicability of USteer to different confidence guidance scores. 
Meanwhile, softmax confidence yields smaller gains than NCI and fDBD. 
This result is consistent with the findings of \citet{liu2026out}, who show that softmax confidence achieves weaker hallucination detection performance on reasoning tasks than fDBD and NCI.

\begin{table}[t] 
\centering 
\caption{ \textbf{USteer improves decoding performance with alternative confidence guidance scores.}
We report accuracy (\%) on \texttt{Qwen-2.5-3B-Instruct}. 
Both fDBD and softmax confidence improve performance over standard inference across all datasets. 
} \label{tab:alternative_metric} 
\renewcommand{\arraystretch}{1.2} 
\vspace{3mm}
\begin{tabular}{lcccc} 
\toprule Method & CSQA & StrategyQA & GSM8K & AQuA \\ \midrule Standard & 77.72 & 68.56 & 79.91 & 54.72 \\ USteer w/ NCI & 79.20 & 70.74 & 80.97 & 55.12 \\ USteer w/ fDBD & 78.21 & 71.18 & 81.20 & 55.12 \\ USteer w/ softmax & 77.89 & 69.43 & 80.52 & 55.12 \\ \bottomrule \end{tabular} 
\end{table}

\paragraph{Limitations.}
While USteer is highly effective for open-weight models, it requires access to intermediate hidden states and gradients, which limits its applicability to black-box proprietary APIs.
In addition, our experiments focus on modern Transformer architectures with residual connections, which form a ``weak-to-strong'' ensemble and motivate the use of intermediate hidden states for uncertainty-guided steering. 
For potential future architectures without residual connections, we still expect USteer to provide benefits, as intermediate layer representations may continue to contain informative uncertainty signals. 
However, without the ensemble structure enabled by residual connections, these uncertainty signals may be less reliable, potentially reducing the effectiveness of USteer.

\paragraph{Broader Impacts.}
Improving the reliability of LLM reasoning is a critical step toward responsible AI deployment. However, users should remain aware that USteer is focused on reducing hallucination and does not inherently address other vital safety concerns, such as bias, fairness, or the generation of harmful content. These factors must be managed through complementary safety frameworks alongside accuracy-enhancing techniques.

\section{Conclusion}

In this work, we introduced USteer, a simple, training-free steering mechanism designed to alleviate hallucinations by leveraging model-internal uncertainty signals during inference. 
Unlike prior steering methods that rely on static displacement vectors derived from supervised probes, USteer utilizes the gradient of a confidence guidance score to dynamically nudge intermediate layer representations toward high-confidence regions of the activation space. By integrating this intervention directly into decoding, we provide a proactive mechanism for hallucination mitigation with negligible overhead. 
Across multiple tasks and settings, we showed that such signals can be used proactively to guide generation toward more reliable outputs.
Our results demonstrate that uncertainty-aware mechanisms can be integrated directly into the generation process, moving beyond post-hoc filtering toward more adaptive and controllable language models.

\newpage

{
\small

\bibliography{example_paper}
\bibliographystyle{ieeenat_fullname}

}

\newpage
\appendix


\section{Implementation details}\label{app:implementation}

\subsection{Examples of Questions and Model Responses for each Dataset}

\subsubsection{CSQA Example}
\vspace{0.5em}
\begin{mdframed}[backgroundcolor=gray!10]
\textbf{Question:} A revolving door is convenient for two direction travel, but it also serves as a security measure at a what?

(a) bank

(b) library

(c) department store

(d) mall

(e) new york

\textbf{Model response:} The answer should be a place where security is a priority. Of the above choices, a bank is the most likely to use a revolving door as a security measure. So the answer is (a).

\end{mdframed}

\subsubsection{StrategyQA Example}
\vspace{0.5em}
\begin{mdframed}[backgroundcolor=gray!10]
\textbf{Question:} Yes or no: Will the Albany in Georgia reach a hundred thousand occupants before the one in New York?

\textbf{Model response:} Albany, New York, already has a population exceeding 100,000. Albany, Georgia, has a much smaller population and is unlikely to surpass the population of Albany, New York. Thus, the Albany in Georgia will not reach a hundred thousand occupants before the one in New York. So the answer is no.

\end{mdframed}

\subsubsection{GSM8K Example}
\vspace{0.5em}
\begin{mdframed}[backgroundcolor=gray!10]
\textbf{Question:} Janet’s ducks lay 16 eggs per day. She eats three for breakfast every morning and bakes muffins for her friends every day with four. She sells the remainder at the farmers' market daily for \$2 per fresh duck egg. How much in dollars does she make every day at the farmers' market?

\textbf{Model response:} 
Janet’s ducks lay 16 eggs per day. She uses 3 for breakfast and 4 for baking, which totals 3 + 4 = 7 eggs used daily. The remaining eggs are 16 - 7 = 9. She sells these for \$2 each, so 9 x 2 = 18. The answer is 18.

\end{mdframed}

\subsubsection{AQuA Example}
\vspace{0.5em}
\begin{mdframed}[backgroundcolor=gray!10]
\textbf{Question:} A password needs to contain 2 letters and 3 numbers. How many different passwords are possible if repetition of letters and numbers is allowed??

(a) 676000

(b) 676

(c) 100

(d) 6760

(e) 25

\textbf{Model response:}  There are 26 letters in the alphabet and 10 digits (0-9). Since repetition is allowed, there are 26 * 26 = 676 ways to choose the letters and 10 * 10 * 10 = 1000 ways to choose the numbers. The total number of different passwords is 676 * 1000 = 676000. The answer is (a).

\end{mdframed}

\subsection{Chain-of-Thought Few-shot Prompting}

Following \cite{wei2022chain}, we use chain-of-thought few-shot prompting to elicit reasoning in the model’s responses.
For all datasets, we adopt the canonical examples from \cite{wei2022chain}.
These prompts guide the model to produce a formatted answer after completing the reasoning process.
For answer extraction, we extract the answer choice following \emph{``So the answer is”} for CSQA, the yes/no answer following \emph{``So the answer is”} for StrategyQA, the number answer following \emph{``The answer is”} for GSM8K, and the answer choice following \emph{``The answer is”} for AQuA.
An answer is considered correct if the extracted answer matches the ground truth.

\section{{Scalability to Larger Language Model}}\label{app:large_model}

\subsection{Evaluation on \texttt{Qwen-2.5-14B-Instruct}.}

To evaluate USteer beyond the sub-8B model scale considered in our main experiments, we conduct additional experiments on \texttt{Qwen-2.5-14B-Instruct} to compare USteer with standard decoding and existing inference-time intervention methods. 
Specifically, we steer last three intermediate layers. 
We use $\alpha=0.7$ for CSQA, $\alpha=1.1$ for StrategyQA and GSM8K, and $\alpha=1.5$ for AQuA. 
The remaining experimental setup follows Section~\ref{sec:main_results}. 

As shown in Table~\ref{tab:14b_results}, USteer improves reasoning accuracy over standard decoding across all four datasets. It achieves the best performance on CSQA and AQuA and remains competitive with existing inference-time intervention methods on StrategyQA and GSM8K. These results demonstrate that USteer remains effective for larger language models. 

\subsection{Evaluation on \texttt{Qwen-3-32B}.}

We further scale our evaluation to \texttt{Qwen-3-32B}~\citep{yang2025qwen3technicalreport}.
As shown in Table~\ref{tab:large_results}, the model achieves substantially higher accuracy across datasets, leaving limited headroom for further improvement compared to smaller models. 
Nevertheless, USteer consistently improves reasoning accuracy across CSQA, StrategyQA, GSM8K, while maintaining accuracy on the challenging task of AQuA.
Overall, these results provide further evidence that uncertainty-based steering remains effective as model scale increases.

\begin{table*}[t]
\caption{
\textbf{USteer consistantly mitigates hallucination and improves reasoning accuracy on \texttt{Qwen-2.5-14B-Instruct}.}
Performance is reported in reasoning task accuracy.
The best performance is shown in \textbf{bold}, and the second-best performance is \underline{underlined}.
}
\vspace{2mm}
\centering
\begin{tabularx}{0.80\linewidth}{
>{\centering\arraybackslash}X
Z{0.1}
Z{0.1}
Z{0.1}
Z{0.1}
}
\toprule
Methods & CSQA & StrategyQA & GSM8K & AQuA \\
\midrule
Standard
& \underline{84.44}
& 72.49
& 91.51
& 75.20 \\
DoLa
& 84.03
& 73.36
& \textbf{92.19}
& \underline{75.98} \\
ITI
& 84.19
& \textbf{75.98}
& 91.58
& 73.80 \\
TruthX
& \underline{84.44}
& 74.24
& \underline{91.81}
& 74.24 \\
\rowcolor{lightgray}
USteer (Ours)
& \textbf{85.26}
& \underline{75.55}
& 91.74
& \textbf{77.17} \\
\bottomrule
\end{tabularx}
\label{tab:14b_results}
\vspace{2mm}
\end{table*}

\begin{table*}[h]
\centering
\caption{ \textbf{USteer consistently mitigates hallucination and improves reasoning accuracy on \texttt{Qwen-3-32B}.
}
We report task accuracy.
USteer consistently improves standard decoding across datasets.
}
\begin{tabular}{Z{0.20} Z{0.12} Z{0.12} Z{0.12} Z{0.1}}
\toprule
Methods & CSQA & StrategyQA & GSM8K & AQuA \\
\midrule 
Standard                    & 86.49 & 78.60 & 94.54 & 77.17\\
USteer                & 86.90 & 79.04 & 95.00 & 77.17\\
\bottomrule
\end{tabular}
\label{tab:large_results}
\end{table*}


\section{Statistical Analysis under Greedy Decoding} \label{app:greedy_ci} 

Greedy decoding is deterministic and therefore does not exhibit variation across repeated runs. Nevertheless, uncertainty in the reported accuracy arises from evaluating the methods on a finite set of test examples. To quantify this sampling uncertainty, we compute 95\% bootstrap confidence intervals by resampling examples from each evaluation set.

As shown in Table~\ref{tab:greedy_ci}, the confidence intervals are relatively wide compared with the differences among methods, precluding conclusive claims of statistical significance under greedy decoding. 
The interval widths are largely determined by dataset size: they are approximately $\pm 6$ percentage points for StrategyQA and AQuA, which contain 229 and 254 examples, respectively, compared with approximately $\pm 3$ percentage points or less for CSQA and GSM8K, which contain 1,221 and 1,319 examples. 
Thus, the statistical power of these comparisons is primarily limited by the sizes of the existing benchmark evaluation sets. 
Nevertheless, the point estimates show that USteer consistently improves over standard greedy decoding across all datasets and both models. 
These results provide complementary evidence for the effectiveness of USteer across decoding settings.

\begin{table*}[t]
\centering
\caption{
\textbf{Accuracy (\%) with 95\% bootstrap confidence intervals under greedy decoding.}
The confidence intervals quantify sampling uncertainty arising from the finite evaluation sets.
Their widths are largely determined by dataset size, with wider intervals observed on the smaller StrategyQA and AQuA evaluation sets.
}
\label{tab:greedy_ci}
\vspace{2mm}
\renewcommand{\arraystretch}{1.2}
\setlength{\tabcolsep}{12pt}
\begin{tabular}{lcccc}
\toprule
Method
& CSQA
& StrategyQA
& GSM8K
& AQuA \\
& ($n=1221$)
& ($n=229$)
& ($n=1319$)
& ($n=254$) \\
\midrule

\multicolumn{5}{c}{\textit{Model: Llama-3.2-1B-Instruct}} \\[2pt]
Standard
& $50.94 \pm 2.78$
& $58.52 \pm 6.55$
& $34.19 \pm 2.58$
& $37.80 \pm 5.91$ \\
DoLa
& $51.19 \pm 2.83$
& $60.70 \pm 6.33$
& $36.39 \pm 2.58$
& $31.50 \pm 5.51$ \\
ITI
& $54.46 \pm 2.87$
& $59.39 \pm 6.33$
& $35.56 \pm 2.54$
& $37.80 \pm 5.91$ \\
TruthX
& $53.07 \pm 2.78$
& $61.57 \pm 6.33$
& $34.87 \pm 2.58$
& $36.22 \pm 5.91$ \\
\rowcolor{lightgray}
USteer (Ours)
& $54.05 \pm 2.78$
& $63.32 \pm 6.11$
& $36.24 \pm 2.58$
& $39.37 \pm 5.91$ \\

\midrule

\multicolumn{5}{c}{\textit{Model: Qwen-2.5-3B-Instruct}} \\[2pt]
Standard
& $77.72 \pm 2.29$
& $68.56 \pm 6.11$
& $79.91 \pm 2.16$
& $54.72 \pm 6.10$ \\
DoLa
& $78.54 \pm 2.29$
& $66.81 \pm 6.11$
& $78.85 \pm 2.24$
& $54.72 \pm 6.30$ \\
ITI
& $77.72 \pm 2.33$
& $68.12 \pm 6.11$
& $80.89 \pm 2.08$
& $57.09 \pm 6.10$ \\
TruthX
& $78.62 \pm 2.29$
& $67.25 \pm 6.11$
& $80.06 \pm 2.16$
& $55.12 \pm 6.10$ \\
\rowcolor{lightgray}
USteer (Ours)
& $79.20 \pm 2.29$
& $70.74 \pm 5.90$
& $80.97 \pm 2.12$
& $55.12 \pm 6.30$ \\

\bottomrule
\end{tabular}
\end{table*}

\section{Proof of Lemma 3.2}~\label{app:proof}

\begin{proof}
Let $u = (\bm{w}_{\hat{c}}^\top \bm{z})^2$ and $v = \|\bm{z}\|_2^2$. 

By the quotient rule, $\nabla g = \frac{v\nabla u - u\nabla v}{v^2}$. 

Substituting $\nabla u = 2(\bm{w}_{\hat{c}}^\top \bm{z})\bm{w}_{\hat{c}}$ and $\nabla v = 2\bm{z}$, we have:
\begin{equation*}
\begin{aligned}
& \nabla_{\bm{z}} g & = &\frac{\|\bm{z}\|_2^2 [2(\bm{w}_{\hat{c}}^\top \bm{z})\bm{w}_{\hat{c}}] - (\bm{w}_{\hat{c}}^\top \bm{z})^2 [2\bm{z}]}{\|\bm{z}\|_2^4} \\
& & = & \frac{2(\bm{w}_{\hat{c}}^\top \bm{z})}{\|\bm{z}\|_2^2} \left( \bm{w}_{\hat{c}} - \frac{\bm{w}_{\hat{c}}^\top \bm{z}}{\|\bm{z}\|_2^2} \bm{z} \right)
\end{aligned}
\end{equation*}
\end{proof}

\section{fDBD Confidence Guidance Score.}\label{app:fdbd_score}

As an alternative to NCI, we use the fDBD score introduced for
out-of-distribution detection and subsequently adapted for hallucination
detection~\citep{liu2026out}. Let $\mathcal{C}_k(\bm{z})$ denote the $k$
highest-scoring alternative tokens under the language head, excluding the
most likely token $\hat{c}$. Following the simplified, squared formulation in
Definition~\ref{def:nci}, we define the fDBD confidence guidance score as
\begin{equation}
g_{\mathrm{fDBD}}(\bm{z})
=
\frac{1}{k}
\sum_{c \in \mathcal{C}_k(\bm{z})}
\frac{
    \big((\bm{w}_{\hat{c}}-\bm{w}_{c})^\top\bm{z}\big)^2
}{
    \lVert\bm{w}_{\hat{c}}-\bm{w}_{c}\rVert_2^2
    \lVert\bm{z}\rVert_2^2
}.
\label{eq:fdbd_guidance}
\end{equation}
We use $k=1000$ in all fDBD experiments. 
A higher score indicates greater
average separation from the decision boundaries associated with competing
tokens and thus higher confidence.

\section{Quantitative Evaluation of Intermediate-layer Uncertainty Signals} \label{app:auroc}

In addition to the qualitative evidence in Figure~\ref{fig:observation}, we quantify the ability of intermediate-layer confidence guidance scores to predict generation correctness. 
Specifically, we compute AUROC using $g(\bm{z}^l)$ from Definition~\ref{def:nci} as the prediction score and generation correctness as the binary target. 
A higher AUROC indicates stronger discrimination between correct and incorrect generations, with $50\%$ corresponding to chance-level performance. 
Table~\ref{tab:guidance_auroc} reports AUROC across four reasoning datasets and the final three layers of \texttt{Llama-3.2-1B-Instruct}, \texttt{Qwen-2.5-3B-Instruct}, and \texttt{Qwen-2.5-14B-Instruct}. 
Although the strength of the signal varies across models, datasets, and layers, the AUROC results support the distributional trend observed in Figure~\ref{fig:observation}: uncertainty signals predictive of generation correctness can emerge before the final layer. 
This further motivate using intermediate uncertainty signals to guide model during inference.

\begin{table*}[t] 
\centering 
\setlength{\tabcolsep}{3.5pt} 
\renewcommand{\arraystretch}{1.2}
\begin{tabular}{lccc} 
\toprule  
& \textbf{\texttt{Llama-3.2-1B-Instruct}} & \textbf{\texttt{Qwen-2.5-3B-Instruct}} & \textbf{\texttt{Qwen-2.5-14B-Instruct}} \\ 
& \textbf{L15 / L14 / L13} & \textbf{L35 / L34 / L33} & \textbf{L47 / L46 / L45} \\
\midrule GSM8K & {74.9} / 69.2 / 65.3 & 75.4 / 75.3 / 65.9 & 71.2 / 74.3 / 62.8 \\ 
CSQA & 60.8 / 56.5 / 55.3 & 62.8 / 58.9 / 57.6 & 69.1 / 68.7 / 65.5 \\ 
StrategyQA & 52.3 / 50.8 / 51.6 & 61.3 / 55.2 / 57.1 & 74.6 / 73.7 / 67.7 \\ 
AQuA & 50.3 / 52.9 / 53.8 & 76.5 / 69.6 / 62.3 & 75.6 / 77.9 / 62.9 \\ 
\bottomrule 
\end{tabular} 
\caption{ \textbf{Intermediate-layer uncertainty guidance scores are predictive of generation correctness across models, datasets, and layers.} The layer-wise uncertainty guidance score $g(\bm{z}^l)$ from Definition~\ref{def:nci} is used as the prediction score, with correct generations treated as the positive class. Each entry reports AUROC for the final three layers, ordered from the final layer to the two preceding layers. A higher AUROC indicates stronger discrimination between correct and incorrect generations.} 
\label{tab:guidance_auroc} 
\end{table*}

\section{Qualitative Analysis of Generated Reasoning} \label{app:qualitative_example} Qualitatively, reasoning generated with USteer can be more focused on the constraints and causal implications expressed in the question. In the example below, standard decoding focuses on an association between pets and kennels, leading to an incorrect answer. In contrast, USteer reasons that obtaining only the puppy implies a constraint on the number of pets and selects the correct answer.

\vspace{0.5em}
\begin{mdframed}[backgroundcolor=gray!10]
\paragraph{Question.} She wanted a kitten and a puppy, so why did she get only the puppy? \begin{quote} (A) one choice for pet \quad (B) cute \quad (C) kennel \quad (D) soft \quad (E) waxy \end{quote} \noindent \textbf{Correct answer:} (A) \paragraph{Generation without USteer.}  The answer must be a reason for getting a pet. Of the above choices, only kennel is a place where pets are kept. So the answer is (C).  \paragraph{Generation with USteer.}  The answer should be that the person wanted both pets but was limited by her circumstances. Of the above choices, the closest reason would be because she only had space for one pet. So the answer is (A). 
\end{mdframed}

This illustrates how USteer can produce a more relevant reasoning trajectory by identifying the implicit constraint underlying the question, rather than following a superficial association with pets.

\section{Effect of USteer on Stochastic Generation Behavior} \label{app:generation_behavior} 

We further examine how USteer affects stochastic generation behavior. 
As shown in Tables~\ref{tab:top1_selection} and~\ref{tab:trajectory_confidence}, USteer modestly increases selection of the most likely token and produces trajectories with higher average predictive probabilities. USteer also reduces the average number of distinct answers across five sampling runs (Table~\ref{tab:answer_diversity}), indicating more consistent generations. 
However, USteer is not simply equivalent to reverting to greedy or lower-temperature decoding. 
Average generation length increases similarly with temperature both with and without USteer (Table~\ref{tab:generation_length}), suggesting that USteer preserves the temperature-dependent behavior of stochastic decoding while guiding generation toward more confident trajectories.

\begin{table}[!htbp] \centering \caption{ \textbf{USteer modestly increases the selection of the most likely token.} We report the percentage of generation steps at which the token with the highest predictive probability is selected under different sampling temperatures. } \label{tab:top1_selection} \vspace{2mm} \renewcommand{\arraystretch}{1.2} \setlength{\tabcolsep}{9pt} \begin{tabular}{lcccc} \toprule Temperature & 0.2 & 0.5 & 0.8 & 1.0 \\ \midrule Standard & 96.3\% & 90.6\% & 83.7\% & 77.6\% \\ USteer & 97.1\% & 92.4\% & 86.4\% & 81.1\% \\ \bottomrule \end{tabular} \end{table}

\begin{table}[!htbp] \centering \caption{ \textbf{USteer produces more confident generation trajectories.} We report the average predictive probability of the sampled tokens under different sampling temperatures. } \label{tab:trajectory_confidence} \vspace{2mm} \renewcommand{\arraystretch}{1.2} \setlength{\tabcolsep}{9pt} \begin{tabular}{lcccc} \toprule Temperature & 0.2 & 0.5 & 0.8 & 1.0 \\ \midrule Standard & 0.767 & 0.751 & 0.718 & 0.675 \\ USteer & 0.800 & 0.786 & 0.756 & 0.719 \\ \bottomrule \end{tabular} \end{table}

\begin{table}[!htbp] \centering \caption{ \textbf{USteer produces more consistent answers across repeated sampling runs.} We report the average number of distinct answers generated across five runs for each question. Lower values indicate less variation among the generated answers. } \label{tab:answer_diversity} \vspace{2mm} \renewcommand{\arraystretch}{1.2} \setlength{\tabcolsep}{9pt} \begin{tabular}{lcccc} \toprule Temperature & 0.2 & 0.5 & 0.8 & 1.0 \\ \midrule Standard & 1.52 & 1.86 & 2.15 & 2.37 \\ USteer & 1.36 & 1.68 & 1.96 & 2.20 \\ \bottomrule \end{tabular} \end{table}

\begin{table}[!htbp] \centering \caption{ \textbf{Generation length exhibits a similar temperature-dependent trend with and without USteer.} We report the average generation length under greedy decoding ($T=0$) and stochastic decoding at different sampling temperatures. } \label{tab:generation_length} \vspace{2mm} \renewcommand{\arraystretch}{1.2} \setlength{\tabcolsep}{8pt} \begin{tabular}{lccccc} \toprule Temperature & 0.0 & 0.2 & 0.5 & 0.8 & 1.0 \\ \midrule Standard & 34.81 & 34.89 & 35.25 & 36.01 & 37.16 \\ USteer & 35.89 & 35.81 & 35.96 & 36.42 & 37.83 \\ \bottomrule \end{tabular} \end{table}

\newpage

\FloatBarrier
\newpage
\input{checklist.tex}

\end{document}

%% file: checklist.tex
\section*{NeurIPS Paper Checklist}

\begin{enumerate}

\item {\bf Claims}
    \item[] Question: Do the main claims made in the abstract and introduction accurately reflect the paper's contributions and scope?
    \item[] Answer: \answerYes{} 
    \item[] Justification: We introduce a training-free steering method that alleviate hallucination in reasoning. 

\item {\bf Limitations}
    \item[] Question: Does the paper discuss the limitations of the work performed by the authors?
    \item[] Answer: \answerYes{} 
    \item[] Justification: We include a "Limitations" section.

\item {\bf Theory assumptions and proofs}
    \item[] Question: For each theoretical result, does the paper provide the full set of assumptions and a complete (and correct) proof?
    \item[] Answer: \answerYes{} 
    \item[] Justification: We provide the full set of assumptions and a complete (and correct) proof for each theoretical result. 

    \item {\bf Experimental result reproducibility}
    \item[] Question: Does the paper fully disclose all the information needed to reproduce the main experimental results of the paper to the extent that it affects the main claims and/or conclusions of the paper (regardless of whether the code and data are provided or not)?
    \item[] Answer: \answerYes{} 
    \item[] Justification: We fully disclose all the information for reproduction. 

\item {\bf Open access to data and code}
    \item[] Question: Does the paper provide open access to the data and code, with sufficient instructions to faithfully reproduce the main experimental results, as described in supplemental material?
    \item[] Answer: \answerNo{} 
    \item[] Justification: While the paper uses publicly available datasets and model, the code is not currently released, as the release is subject to legal and compliance review; therefore, full open access is not provided at submission time.

\item {\bf Experimental setting/details}
    \item[] Question: Does the paper specify all the training and test details (e.g., data splits, hyperparameters, how they were chosen, type of optimizer) necessary to understand the results?
    \item[] Answer: \answerYes{} 
    \item[] Justification: Experiments are detailed in section Experiments and Appendix. 

\item {\bf Experiment statistical significance}
    \item[] Question: Does the paper report error bars suitably and correctly defined or other appropriate information about the statistical significance of the experiments?
    \item[] Answer: \answerYes{} 
    \item[] Justification: Some of our experiments are deterministic, using greedy decoding with standard datasets and off-the-shelf models. For experiments involving stochasticity, we report 95\% confidence intervals to assess statistical significance.

\item {\bf Experiments compute resources}
    \item[] Question: For each experiment, does the paper provide sufficient information on the computer resources (type of compute workers, memory, time of execution) needed to reproduce the experiments?
    \item[] Answer: \answerYes{} 
    \item[] Justification: We report the GPU type and computational time in "Experiments" section.
    
\item {\bf Code of ethics}
    \item[] Question: Does the research conducted in the paper conform, in every respect, with the NeurIPS Code of Ethics \url{https://neurips.cc/public/EthicsGuidelines}?
    \item[] Answer: \answerYes{} 
    \item[] Justification: The research conducted in the paper conforms, in every respect, with the NeurIPS Code of Ethics. 

\item {\bf Broader impacts}
    \item[] Question: Does the paper discuss both potential positive societal impacts and negative societal impacts of the work performed?
    \item[] Answer: \answerYes{} 
    \item[] Justification: We include a "Broader Impacts" section.
    
\item {\bf Safeguards}
    \item[] Question: Does the paper describe safeguards that have been put in place for responsible release of data or models that have a high risk for misuse (e.g., pre-trained language models, image generators, or scraped datasets)?
    \item[] Answer: \answerNA{} 
    \item[] Justification: The paper poses no such risks.

\item {\bf Licenses for existing assets}
    \item[] Question: Are the creators or original owners of assets (e.g., code, data, models), used in the paper, properly credited and are the license and terms of use explicitly mentioned and properly respected?
    \item[] Answer: \answerYes{} 
    \item[] Justification: All licenses are respected.

\item {\bf New assets}
    \item[] Question: Are new assets introduced in the paper well documented and is the documentation provided alongside the assets?
    \item[] Answer: \answerNA{} 
    \item[] Justification: The paper does not release new assets.

\item {\bf Crowdsourcing and research with human subjects}
    \item[] Question: For crowdsourcing experiments and research with human subjects, does the paper include the full text of instructions given to participants and screenshots, if applicable, as well as details about compensation (if any)? 
    \item[] Answer: \answerNA{} 
    \item[] Justification: The paper does not involve crowdsourcing nor research with human subjects.

\item {\bf Institutional review board (IRB) approvals or equivalent for research with human subjects}
    \item[] Question: Does the paper describe potential risks incurred by study participants, whether such risks were disclosed to the subjects, and whether Institutional Review Board (IRB) approvals (or an equivalent approval/review based on the requirements of your country or institution) were obtained?
    \item[] Answer: \answerNA{} 
    \item[] Justification: The paper does not involve crowdsourcing nor research with human subjects.

\item {\bf Declaration of LLM usage}
    \item[] Question: Does the paper describe the usage of LLMs if it is an important, original, or non-standard component of the core methods in this research? Note that if the LLM is used only for writing, editing, or formatting purposes and does \emph{not} impact the core methodology, scientific rigor, or originality of the research, declaration is not required.
    \item[] Answer: \answerNA{} 
    \item[] Justification: The core method development in this research does not involve LLMs as any important, original, or non-standard components.

\end{enumerate}